\documentclass[lettersize,journal]{IEEEtran}
\usepackage{amsmath,amsfonts}
\usepackage{algorithmic}
\usepackage{algorithm}
\usepackage{array}
\usepackage[caption=false,font=normalsize,labelfont=sf,textfont=sf]{subfig}
\usepackage{textcomp}
\usepackage{stfloats}
\usepackage{url}
\usepackage{verbatim}
\usepackage{graphicx}
\usepackage{cite}
\usepackage{amsthm}
\usepackage{accents}
\newcommand{\ubar}[1]{\underaccent{\bar}{#1}}
\usepackage{nomencl}
\makenomenclature
\usepackage{hyperref}
{

      \newtheorem{definition}{Definition}
      \newtheorem{theorem}{\textbf{Theorem}}[section]
      \newtheorem{lemma}{\textbf{Lemma}}[section]
      \newtheorem{corollary}{Corollary}[section]
}

\begin{document}

\title{Decentralized Distributed Expert Assisted Learning (D2EAL) approach for cooperative target-tracking}

\author{Shubhankar Gupta$^{*}$ and Suresh Sundaram$^{*}$
\thanks{$^{*}$Shubhankar Gupta ({\tt\small shubhankarg@iisc.ac.in}) and Suresh Sundaram ({\tt\small vssuresh@iisc.ac.in}) are affiliated with the Artificial Intelligence and Robotics Lab (AIRL), Department of Aerospace Engineering, Indian Institute of Science, Bengaluru, Karnataka, India}}

\maketitle

\begin{abstract}
This paper addresses the problem of cooperative target tracking using a heterogeneous multi-robot system, where the robots are communicating over a dynamic communication network, and heterogeneity is in terms of different types of sensors and prediction algorithms installed in the robots. The problem is cast into a distributed learning framework, where robots are considered as \lq agents' connected over a dynamic communication network. Their prediction algorithms are considered as \lq experts' giving their look-ahead predictions of the target's trajectory. In this paper, a novel Decentralized Distributed Expert-Assisted Learning (D2EAL) algorithm is proposed, which improves the overall tracking performance by enabling each robot to improve its look-ahead prediction of the target's trajectory by its information sharing, and running a weighted information fusion process combined with online learning of weights based on a prediction loss metric. Theoretical analysis of D2EAL is carried out, which involves the analysis of worst-case bounds on cumulative prediction loss, and weights convergence analysis. Simulation studies show that in adverse scenarios involving large dynamic bias or drift in the expert predictions, D2EAL outperforms well-known covariance-based estimate/prediction fusion methods, both in terms of prediction performance and scalability. 
\end{abstract}

\begin{IEEEkeywords}
Distributed Online Learning, Decentralized Learning, Cooperative target trajectory prediction, Heterogeneous multi-robot system, Large dynamic prediction bias 
\end{IEEEkeywords}

\section{Introduction}
With the advent of advanced sensor/communication technologies,  computer vision, deep learning algorithms, and reliable robotic platforms (Unmanned aerial, ground, surface, underwater vehicles), there has been an increased interest among researchers in the area of cooperative Multi-Robotic Systems (MRS). Since the associated technologies are getting cheaper, smaller, and more reliable, MRS are attractive for usage in unsafe and uncertain environments \cite{mohiuddin2020survey}. MRS have a wide range of applications such as search and rescue \cite{scherer2015autonomous}, firefighting \cite{harikumar2018multi}, convoy protection \cite{spry2005convoy}, traffic monitoring \cite{khan2020smart}, surveillance, etc., and all these applications involve tracking a target as one of the fundamental tasks.    

One of the important problems in target tracking using a robot equipped with a suitable sensor suite is the look-ahead trajectory prediction \cite{hao2018review}. Based on predicted trajectory, the robot plans its path to ensure desirable tracking performance. Thus, a target tracking problem can be divided into three phases forming a loop: prediction, path planning, and control. Instead of using a single robot, having multiple robots tracking the target can be superior in better predicting the target's future trajectory while providing more comprehensive coverage of the search area that is being monitored for the target tracking purpose \cite{khan2016cooperative}. Especially, the cooperative multi-robot scenarios where each robot's prediction of the target trajectory is different, owing to heterogeneity in terms of the sensor suite and/or prediction algorithm, can lead to more robust tracking performance \cite{rizk2019cooperative}. 

The literature on target tracking using MRS mainly deals with the path planning and/or control aspects of the tracking task. Most of the recent works address target tracking either by coordinated control via formation flying \cite{ma2015cooperative}, \cite{sun2018collaborative} or region based approach \cite{jung2006cooperative}. In \cite{subbarao2017target}, cooperative tracking is addressed in a dynamically changing communication network topology. Here, the tracking problem is converted into cooperative control using pinning control technique and consensus on target states, assuming that only pinned UAVs know the target location \cite{wang2014pinning}. In \cite{wang2011cooperative}, distributed Kalman filter is used for target's position estimation, and distributed flocking control for tracking and collision avoidance. In \cite{hausman2016cooperative}, a centralized cooperative control algorithm for target tracking is presented which involves target position estimation from onboard sensing in each of the UAVs. 

Considering the prediction aspect of the tracking task, it is beneficial for the robots in a MRS to have installed sensors and prediction algorithms that complement each other \cite{rizk2019cooperative}. This nature of heterogeneity in robots due to different sensor suite, prediction algorithms, and environmental uncertainty, thus, influences the accuracy of target's trajectory prediction by individual robots. Hence, there is a need to develop a cooperative information fusion algorithm that minimizes the uncertainty in target's trajectory prediction by each robot in the MRS. 

In MRS applications involving multi-sensor fusion, distributed multiple estimate/prediction fusion is mainly performed using the well-known fusion methods like Kalman Filter/Fusion (KF) \cite{maybeck1982stochastic}, \cite{uhlmann2003covariance}, Covariance Intersection (CI) \cite{matzka2009comparison}, \cite{julier2017general}, and Covariance Union (CU) \cite{matzka2009comparison}, \cite{reece2010generalised}. \cite{weng2012bayesian} proposes a Bayes framework based Fusion (BF) algorithm which outperforms the CI in simulations. Based on the CI algorithm, \cite{carrillo2013decentralized} proposes an approximate decentralized multi-robot cooperative localization algorithm, with reduced processing and communication costs, thereby maintaining consistency while handling asynchronous communication constraints. In \cite{assa2015kalman}, a nonlinear KF-based sensor fusion framework is proposed which is based on an adaptation technique that compensates system noise variations, and an iterative scheme that deals with the fast system dynamics. \cite{chang2021resilient} utilizes CI explicitly in the communication update of their proposed multi-robot localization in order to ensure estimation consistency and enhance resilience. In \cite{daass2021design}, three different data fusion architectures based on the KF and the CI are studied, and it is shown that the partially distributed architecture exhibits best stability, and lowest computing and communication costs. \cite{wang2021fault} proposes a fully decentralized multi-robot cooperative localization algorithm based on CU, where CU is used to handle spurious sensor data in the fusion process to make sure the fused estimates stay consistent. All of these covariance-based fusion methods usually involve assumptions regarding consistency and correlation among the multiple estimates being fused; KF requires the estimates to be uncorrelated, CI and BF require that the estimates being fused are consistent, and CU requires one of the estimates being fused to be consistent. Moreover, covariance-based methods work based on the covariance information of the estimates being fused, thus, requiring their covariance information as an input. In adverse scenarios involving dynamic and potentially large biases or drift in the estimates/predictions being fused, these covariance-based methods may not perform satisfactorily, or may even fail. Thus, there is a need for algorithms which are designed to effectively handle large dynamic biases or drift in the estimates/predictions, and do not require any covariance information of the predictions being fused. 

In this paper, the problem of cooperatively tracking a target using heterogeneous MRS is formulated into a distributed online learning framework inspired by the works in distributed learning \cite{sahu2016distributed, sahu2017dist}. Both of these works propose a centralized distributed learning framework in which multiple agents aim to predict a sequence or a signal while communicating over a network. In practice, centralized framework is not scalable, and can be infeasible due to limited communication bandwidth, communication channel size constraints, and/or information storage constraints. Moreover, such a centralized agent acts as the single point of failure for the multi-agent system. Thus, we cast the problem of cooperative target tracking using a heterogeneous MRS into a decentralized distributed learning framework. Here, robots are considered as `agents' connected over a dynamic communication network. A prediction algorithm is present in each of the robots as an `expert', providing look-ahead prediction of the target's trajectory to the robot. The expert predictions can be different due to differences in sensors, prediction algorithms, and/or environmental uncertainty. We propose a Decentralised Distributed Expert Assisted Learning (D2EAL) algorithm enabling each robot to cooperatively track the target accurately. The decentralized nature of D2EAL handles scalability issues and avoids any single point of failure occurrences, thus bringing resilience into the MRS. D2EAL involves robots learning from their past prediction experiences including those shared by their neighbouring robots in the communication network, while utilizing a weighted information fusion process, thereby improving tracking accuracy. The relative weights are updated based on an exponential weight scheme similar to \cite{cesa2006prediction}.   

D2EAL is analyzed theoretically in terms of worst case upper bounds on the cumulative loss incurred by each robot, and the worst-case bounds are shown to be growing sub-linearly with the time horizon. Further, convergence analysis of the weights in D2EAL is carried out, and it is shown that the weights do converge under certain reasonable assumptions. The performance of D2EAL is then evaluated using a simulated environment with an adverse setting. In this setting, D2EAL is compared against three baseline decentralized fusion methods - Mean, Median, Greedy-Local, and four well-known decentralized fusion methods - KF, CI, BF, and CU. The simulation results clearly indicate that D2EAL outperforms the baseline and the well-known fusion methods, with a substantial margin. Further, a scalability simulation study shows that D2EAL performs significantly better than all these seven fusion methods.  

The rest of this paper is organised as follows: section II presents problem formulation and a novel distributed learning framework for heterogeneous multi-robot target tracking, along with the proposed D2EAL algorithm. Section III presents theoretical analysis of the D2EAL algorithm. Section IV presents results on the proposed algorithm's performance and scalability via two simulation studies. Finally, section V concludes this paper. 
\section{Decentralised Distributed Expert-Assisted Learning}
In this section, we first describe the problem of cooperative target tracking using a heterogeneous Multi-Robot System (MRS). Next, a novel distributed learning framework for cooperatively tracking a target is proposed. Finally, we present the Decentralized Distributed Expert-Assisted Learning (D2EAL) algorithm. 
\nomenclature{\(A_i\)}{$i^{th}$ robot's prediction algorithm}
\nomenclature{\(x_{t,i}\)}{$i^{th}$ robot's 2-D position vector (in $m$)}
\nomenclature{\(\bar{v}_{t,i}\)}{$i^{th}$ robot's body-axis velocity vector ($m/s$)}
\nomenclature{\(\phi_{t,i}\)}{$i^{th}$ robot's heading angle (radians)}
\nomenclature{\(\bar{w}_{t,i}\)}{$i^{th}$ robot's yaw rate ($rad/s$)}
\nomenclature{\(x_{t,T_g}\)}{target's 2-D position vector (in $m$)}
\nomenclature{\(\bar{v}_{t,T_g}\)}{target's body-axis velocity vector ($m/s$)}
\nomenclature{\(\phi_{t,T_g}\)}{target's heading angle (radians)}
\nomenclature{\(\bar{w}_{t,T_g}\)}{target's yaw rate ($rad/s$)}
\nomenclature{\(\bar{v}_{t,i}^R\)}{$i^{th}$ robot's velocity reference command signal}
\nomenclature{\(\Delta \bar{v}_{t,i}\)}{$i^{th}$ robot's velocity correction control signal}
\nomenclature{$\hat{x}_{(t+\tau / t),T_g}^{i}$}{$i^{th}$ robot's $\tau$-step look-ahead prediction of target's position} 
\nomenclature{\(\hat{x}_{(t+\tau / t),T_g}^{A_i}\)}{$\tau$-step look-ahead prediction of target's position given by algorithm $A_i$}
\nomenclature{\(\zeta_{t,i}^{\tau}\)}{drift in algorithm $A_i$'s $\tau$-step look-ahead prediction} 
\nomenclature{\(\{y_t\}_{t=1}^{T}\)}{unknown discrete-time target sequence, $y_t \equiv x_{t,T_g}$} 
\nomenclature{\(f_{t+1,i}\)}{one-step look-ahead prediction of the target sequence as given by the $i^{th}$ expert, $f_{t+1,i} \equiv \hat{x}_{(t+1/t),T_g}^{A_i}$}
\nomenclature{\(f_{t+\tau,i}^p\)}{$\tau$-step look-ahead prediction of the target sequence as given by the $i^{th}$ expert, $f_{t+\tau,i} \equiv \hat{x}_{(t+\tau|t),T_g}^{A_i}$} 
\nomenclature{\(\hat{f}_{t+1,i}\)}{$i^{th}$ agent's one-step look-ahead prediction of the target sequence, $\hat{f}_{t+1,i} \equiv \hat{x}_{(t+1/t),T_g}^{i}$}
\nomenclature{\(\hat{l}_{t,i}\)}{$i^{th}$ agent's prediction loss: $l(\hat{f}_{t,i},y_t)$}
\nomenclature{\(\hat{L}_{t,i}\)}{$i^{th}$ agent's cumulative prediction loss: $\sum_{s=1}^{t} \hat{l}_{s,i}$}
\nomenclature{\(l_{t,i}\)}{$i^{th}$ expert's prediction loss: $l(f_{t,i},y_t)$}
\nomenclature{\(L_{t,i}\)}{$i^{th}$ expert's cumulative prediction loss: $\sum_{s=1}^{t} l_{s,i}$}
\nomenclature{\(\Omega_i(t)\)}{$i^{th}$ agent's neighbour set as per the communication network at time $t$}
\printnomenclature
\subsection{Problem Formulation}
The scenario of cooperative target-tracking with a heterogeneous MRS (shown in Fig.\ref{fig:01}) involves a target whose trajectory is being predicted by multiple heterogeneous robots that cooperate with each other over a communication network. The target dynamics is unknown to the robots.
Each robot is installed with a sensor suite and a data-driven prediction algorithm to predict the target's trajectory based on its sensor information. Further, the target is observable to all the robots. Heterogeneity in the MRS is in terms of different types of sensors and prediction algorithms that are installed in the robots. These prediction algorithms can exhibit different prediction accuracy for different parts of the target's trajectory.    
\begin{figure}[h]
    \centering
    \includegraphics[width=0.45\textwidth]{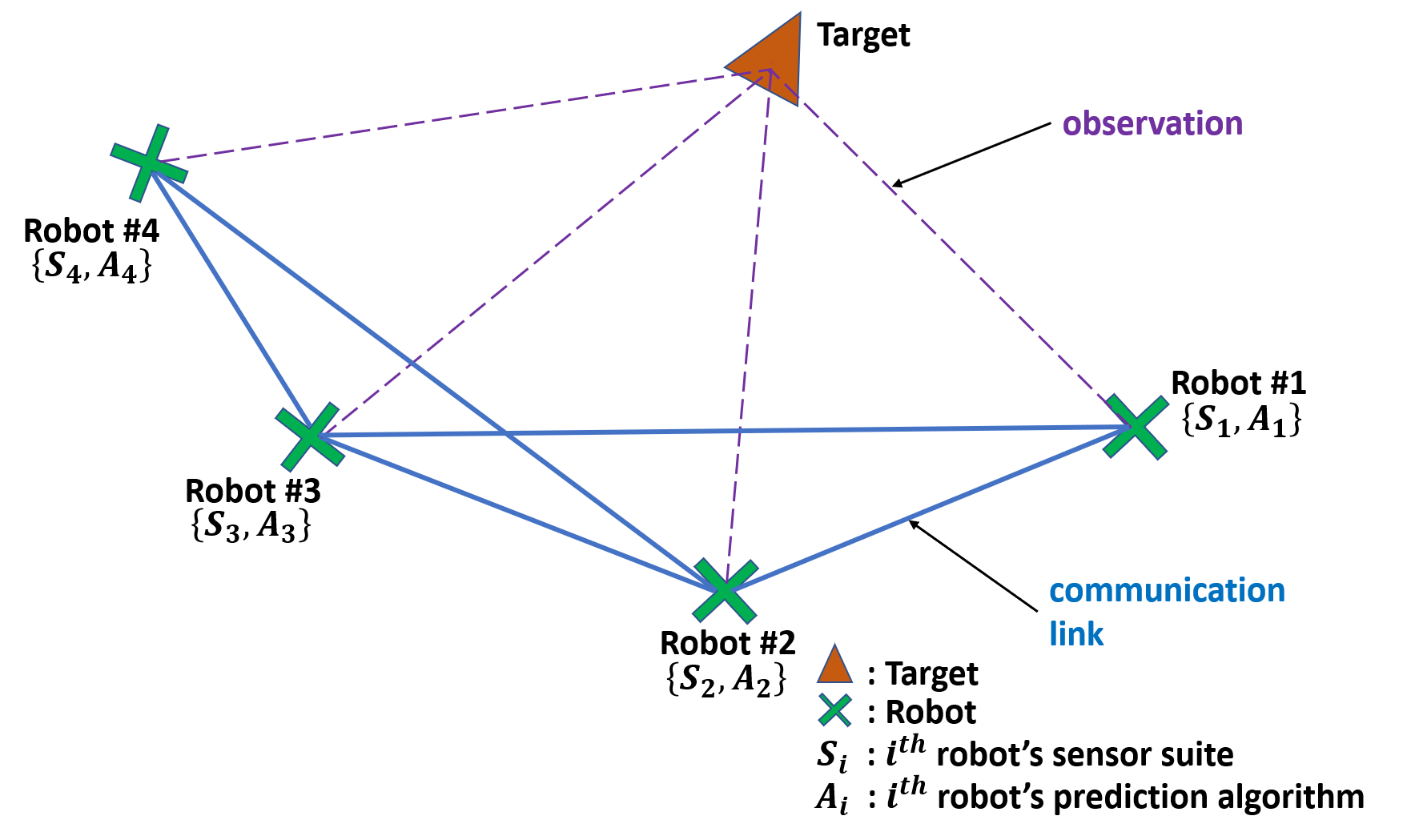}
    \caption{Heterogeneous Multi-Robot Cooperative Target-Tracking Scenario}
    \label{fig:01}
\end{figure}

The robots that are directly connected via a communication channel (or communication link) with each other in pairs can be called neighbouring robots. The topology of the dynamic communication network is represented by an underlying bi-directional dynamic graph $G(t)$, where $t$ is the discrete-time variable. The robots can communicate information with their neighbours only once between two successive observations of the target's location. The robots can infer only from their neighbours and are not aware of the overall communication connectivity graph, i.e., the robots only have local knowledge of the communication network. 

Let $N$ denote the total number of robots in the MRS, and let each robot be represented by its index $i$, where $i \in [N]$. The robots are equipped with a data-driven prediction algorithm that predicts look-ahead trajectory of the target by processing real-time data from the on-board sensors. Lets denote $i^{th}$ robot's prediction algorithm as $A_i$, $\forall i \in [N]$. The collection of algorithms $\{A_i\}_{i=1}^{N}$ is considered to be heterogeneous, i.e., $i^{th}$ robot's algorithm $A_i$ is different from $j^{th}$ robot's algorithm $A_j$, $i\neq j$ and $\forall i,j \in [N]$; the algorithms can be of different class (or type), or same class but different parameters. This implies that the prediction accuracy of these algorithms is likely to be different from each other for different parts of the target's trajectory. 

\textit{Robot Model}:
Consider the following discrete time 3-DOF kinematic model for the $i^{th}$ robot, where $\Delta T$ is the sampling period (seconds), $\forall i \in [N]$
\begin{subequations} \label{eqt01}
\begin{align}
    x_{t+1,i} &= x_{t,i} + \Delta T \begin{bmatrix}
               \cos{\phi_{t,i}} & -\sin{\phi_{t,i}} \\
               \sin{\phi_{t,i}} & \cos{\phi_{t,i}}
               \end{bmatrix} \bar{v}_{t,i} \\
    \phi_{t+1,i} &= \phi_{t,i} + \Delta T \bar{w}_{t,i} 
\end{align}
\end{subequations}
where $x_{t,i} \in \mathbb{R}^2$ is the $i^{th}$ robot's 2-D position vector (in $m$), $\bar{v}_{t,i} \in \mathbb{R}^2$ is the $i^{th}$ robot's body-axis velocity vector ($m/s$), $\phi_{t,i} \in \mathbb{R}$ is the $i^{th}$ robot's heading angle (radians), and $\bar{w}_{t,i} \in \mathbb{R}$ is $i^{th}$ robot's yaw rate ($rad/s$) at discrete-time $t$, respectively. Here, the body-axis velocity $\bar{v}_{t,i}$ and yaw rate $\bar{w}_{t,i}$ act as bounded control inputs for the $i^{th}$ robot. 

\textit{Target Model}: 
The target model is similar to the robot model. The target's position vector $x_{t,T_g} \in \mathbb{R}^2$ (in $m$), heading angle $\phi_{t,T_g} \in \mathbb{R}$ (radians), body-axis velocity $\bar{v}_{t,T_g} \in \mathbb{R}^2$ (m/s), and yaw rate $\bar{w}_{t,T_g} \in \mathbb{R}$ ($rad/s$), respectively, can be represented by replacing $i$ with $T_g$ in the set of equations (\ref{eqt01}). Similarly, $\bar{v}_{t,T_g}$ and $\bar{w}_{t,T_g}$ act as bounded control inputs for the target at time $t$, which are considered unknown to the robots. 
 
 \textit{Translational Control Law}: 
 For the $i^{th}$ robot, the translational control law consists of two terms as given below
 \begin{equation} \label{eqt02.1}
     \bar{v}_{t,i} = \bar{v}_{t,i}^R + \Delta \bar{v}_{t,i}
 \end{equation}
 where $\bar{v}_{t,i}^R$ is the $i^{th}$ robot's reference command signal responsible for chasing the target, and $\Delta \bar{v}_{t,i}$ is the $i^{th}$ robot's correction control signal responsible for avoiding collisions with other robots. 
 
 Denote $R_{t,i} \in \mathbb{R}^{2\times 2}$ as the $i^{th}$ robot's body-global rotation matrix at time $t$, defined as $R_{t,i} = \begin{bmatrix} \cos{\phi_{t,i}} & -\sin{\phi_{t,i}} \\ \sin{\phi_{t,i}} & \cos{\phi_{t,i}} \end{bmatrix}$. 
 
 The $i^{th}$ robot's reference command signal $\bar{v}_{t,i}^R$ is given as
 \begin{equation} \label{eqt02.2}
     \bar{v}_{t,i}^R = k_1 R_{t,i}' \frac{\Delta \hat{x}_{(t+\tau|t),T_g}^{i}}{||\Delta \hat{x}_{(t+\tau|t),T_g}^{i}||} (||\Delta x_{t,T_g}^{i}|| - d_S)
 \end{equation}
 where $(\cdot)'$ represents the transpose operation, $||\cdot||$ is the 2-norm or the Euclidean norm, $k_1 > 0$ is a control parameter. $\Delta x_{t,T_g}^{i} := x_{t,T_g} - x_{t,i}$, where $x_{t,T_g}$ is the target's position vector at time $t$, and $x_{t,i}$ is the $i^{th}$ robot's position vector at time $t$. $d_S > 0$ ($m$) is a parameter indicating the distance each robot should maintain from the target while chasing it. 
  Here, $\Delta \hat{x}_{(t+\tau|t),T_g}^{i}$ is defined as 
 \begin{equation} \label{eqt02.3}
    \Delta \hat{x}_{(t+\tau|t),T_g}^{i} := \hat{x}_{(t+\tau|t),T_g}^{i} - x_{t,i} 
 \end{equation} 
 where $\hat{x}_{(t+\tau|t),T_g}^{i}$ is the $i^{th}$ robot's $\tau$-step look-ahead prediction of target's position at time $t$, and $x_{t,i}$ is the $i^{th}$ robot's position at time $t$.
 
 Further, we assume that each robot is equipped with a collision avoidance system, which makes sure that while chasing the target, robots do not collide with each other. Considering eq.(\ref{eqt02.1}), this behavior can be modeled by the correction control signal $\Delta \bar{v}_{t,i}$ for the $i^{th}$ robot by using an \textit{inter-robot collision avoidance} control law given as follows:
\begin{equation} \label{eqt02.411}
    \Delta \bar{v}_{t,i} = -k_2 R_{t,i}' \frac{x_{t,p_t^i} - x_{t,i}}{||x_{t,p_t^i} - x_{t,i}||^2}
\end{equation}
where $(\cdot)'$ represents the transpose operation, $||\cdot||$ is the 2-norm or the Euclidean norm, $k_2 > 0$ is a control parameter, $p_t^i \in [N]\setminus \{i\}$ is the index of the robot spatially nearest to $i^{th}$ robot at time $t$, formally defined as $p_t^i := \arg \min_{j\in [N]\setminus \{i\}} ||x_{t,j} - x_{t,i}||$. Thus, $x_{t,p_t^i}$ is the position vector of the robot spatially nearest to the $i^{th}$ robot at time $t$. 
 

\textit{Heading Control Law for the $i^{th}$ robot}:
Consider a heading angle requirement for the robots; robots are required to yaw in such a way that their heading direction should point towards their $\tau$-step look-ahead estimate of target's position $\hat{x}_{(t+\tau|t),T_g}^{i}$. The angle between $\Delta \hat{x}_{(t+\tau|t),T_g}^{i}$ (from eq.\ref{eqt02.3}) and the $i^{th}$ robot's heading direction $h_{t,i} = {\begin{bmatrix} \cos \phi_{t,i} & \sin \phi_{t,i} \end{bmatrix}}'$, with respect to the $\Delta \hat{x}_{t+\tau,T_g}^{i}$ direction, can be obtained as $\Delta \phi_{t,err}^{i} = atan2\left(h_{t,i} \times \Delta \hat{x}_{t+\tau,T_g}^{i}, h_{t,i} \cdot \Delta \hat{x}_{t+\tau,T_g}^{i}\right)$, where the first argument involves a cross-product and the second argument involves dot-product. As per the heading angle requirement, $i^{th}$ robot's yaw control law can be given as
\begin{equation} \label{eqt03}
    \bar{w}_{t,i} = k_3 \Delta \phi_{t,err}^{i}
\end{equation}
where $k_3 > 0$ is a control parameter. 

\textit{Abstract Model for Prediction Algorithm of the $i^{th}$ robot}:
Lets denote $\hat{x}_{(t+\tau|t),T_g}^{A_i} \in \mathbb{R}^2$ as the $\tau$-step look-ahead prediction of target's position, given by algorithm $A_i$ at time $t$. We use a simplified model for algorithm $A_i$'s prediction, which reflects the likeliness of its prediction accuracy to be different from other algorithms $\{A_j\}_{\forall j \in [N]\setminus \{i\}}$, as follows:
\begin{equation} \label{eqt04}
        \hat{x}_{(t+\tau|t),T_g}^{A_i} = x_{t+\tau,T_g} + \zeta_{t,i}^{\tau} + \nu_{t,i}^{\tau} 
\end{equation}
where $x_{t+\tau,T_g}$ is the target's true position vector at time $t+\tau$, and $\zeta_{t,i}^{\tau} \in \mathbb{R}^2$ represents drift in algorithm $A_i$'s $\tau$-step look-ahead prediction of target's position, and $\nu_{t,i}^{\tau} \in \mathbb{R}^2$ is zero-mean gaussian prediction noise with covariance $C_{t,i}^{\tau}$, at time $t$. Both the drift and the noise terms model the inaccuracy in the prediction by algorithm $A_i$. The drift term is defined as
\begin{equation} \label{eqt04.01}
     \zeta_{t,i}^{\tau} = c_{t,i}^{\tau} s_t
\end{equation}
where $c_{t,i}^{\tau} \in \mathbb{R}^2$ can be time-varying, and $s_t$ represents the discrete time period for which the drift sustains till time $t$ after the most recent drift reset, defined as
\begin{equation} \label{eqt05}
        s_{t+1} = \left\{
        \begin{array}{ll}
            s_t + 1  & :\quad \textit{with prob.} \quad (1-p) \\
            0 & :\quad \textit{with prob.} \quad p
        \end{array}
    \right.
\end{equation}
where $p$ is the drift reset probability. As per the above model, if $c_{t,i}^{\tau}$ remains constant in time, the drift term $(c_{t,i}^{\tau} s_t)$ grows linearly with time until it gets reset back to zero with probability $p$.

Note that the reference command signal in the translational control law, as given by equations (\ref{eqt02.2}) and (\ref{eqt02.3}), involves the $i^{th}$ robot's $\tau$-step look-ahead prediction of target's position at time $t$, $\hat{x}_{(t+\tau|t),T_g}^{i}$. Since the robot may be using some information fusion strategy, $\hat{x}_{(t+\tau|t),T_g}^{i}$ may not be equal to algorithm $A_i$'s $\tau$-step look-ahead prediction of target's position at time $t$, $\hat{x}_{(t+\tau|t),T_g}^{A_i}$. 

\subsection{Mathematical Formulation}

The problem of cooperative target trajectory tracking using heterogeneous MRS is mathematically formulated into a distributed expert-assisted learning framework. The robots are considered as `agents' in an undirected dynamic communication graph $G(t)$, and the prediction algorithms are considered as `experts' that give their expert prediction of the target behavior (Fig.\ref{fig:02}). Thus, we can call the $i^{th}$ robot as the $i^{th}$ agent, and the $i^{th}$ robot's prediction algorithm $A_i$ as the $i^{th}$ expert, $\forall i \in [N]$; the $i^{th}$ expert can be seen as assisting the $i^{th}$ agent by sharing its prediction with the agent, as shown in Fig. \ref{fig:02}. With the assistance of its expert and its neighbouring agents as per the communication network, each agent aims to predict an unknown discrete-time target sequence $\{y_t\}_{t=1}^{T}$ which is considered to be the target's trajectory $\{x_{t,T_g}\}_{t=1}^{T}$, i.e., $y_t \equiv x_{t,T_g}$, $t=1,2,\cdots,T$, where $T$ is the time horizon.

In this framework, we denote $f_{t+1,i}$ as the one-step look-ahead prediction of the target sequence as given by the $i^{th}$ expert at time $t$. Similarly, denote $f_{t+\tau,i}^p$ as the $\tau$-step look-ahead prediction of the target sequence as given by the $i^{th}$ expert at time $t$. This implies $f_{t+1,i} \equiv \hat{x}_{(t+1|t),T_g}^{A_i}$, and $f_{t+\tau,i}^p \equiv \hat{x}_{(t+\tau|t),T_g}^{A_i}$. Formally, an \lq expert' can be defined as:
\begin{definition}
Algorithms $\{A_i\}_{i=1}^{N}$ can be considered as \lq experts' if and only if their one-step look-ahead predictions $\{f_{t+1,i}\}_{i=1}^{N}$ satisfy
\begin{subequations} \label{eqt06}
\begin{align}
    ||f_{t+1,i} - f_{t+1,j}|| &\leq \delta_{t+1} \\
    \sum_{t=1}^{T} \delta_t \leq \Delta_o
\end{align}
\end{subequations} 
for all $i,j \in [N]$, some positive time varying scalar $\delta_t$, and some positive scalar constant $\Delta_o$, with $||\cdot||$ as the Euclidean norm. 
\end{definition}
Denote $\hat{f}_{t+1,i}$ as the $i^{th}$ agent's one-step look-ahead prediction of the target sequence at time $t$, i.e., $\hat{f}_{t+1,i} \equiv \hat{x}_{(t+1|t),T_g}^{i}$. Similarly, denote $\hat{f}_{t+\tau,i}^p$ as the $i^{th}$ agent's $\tau$-step look-ahead prediction of the target sequence at time $t$, i.e., $\hat{f}_{t+\tau,i}^p \equiv \hat{x}_{(t+\tau|t),T_g}^{i}$. Note that $\hat{x}_{(t+\tau|t),T_g}^{i}$ is used by the $i^{th}$ robot in its control law specified by equations (\ref{eqt02.2}) and (\ref{eqt02.3}). $\hat{f}_{t+1,i}$ and $\hat{f}_{t+\tau,i}^p$ may or may not be equal to $f_{t+1,i}$ and $f_{t+\tau,i}^p$, respectively, as the $i^{th}$ agent need not use the $i^{th}$ expert's prediction directly; instead, $\hat{f}_{t+1,i}$ and $\hat{f}_{t+1,i}^p$ can be the result of some information fusion strategy. We will describe our proposed strategy in the next subsection. Also note that ${f}_{t+1,i} = {f}_{t+1,i}^p$, but ${f}_{t+\tau,i} \neq {f}_{t+\tau,i}^p$ for $\tau = 2,3,\cdots$. Same holds true for $\hat{f}$ predictions. For ease of algorithm understanding and theoretical analysis, the above notation is used throughout in the paper.   

The agent incurs a loss $\hat{l}_{t,i} := l(\hat{f}_{t,i},y_t)$, where $l(\cdot,\cdot) \in [0,1]$ is a convex loss function. The loss due to $i^{th}$ expert's prediction at time $t$ is defined as $l_{t,i} := l(f_{t,i},y_t)$. The neighbour set is defined as: $\Omega_i(t) = \{j : j^{th}$ agent is the neighbour of $i^{th}$ agent at time $t$, as per $G(t)\}$, where $i=1,2,\cdots,N$. Further, define $\Lambda_i(t) := \Omega_i(t) \cup \{i\}$. 
\begin{figure}[h]
    \centering
    \includegraphics[width=0.45\textwidth]{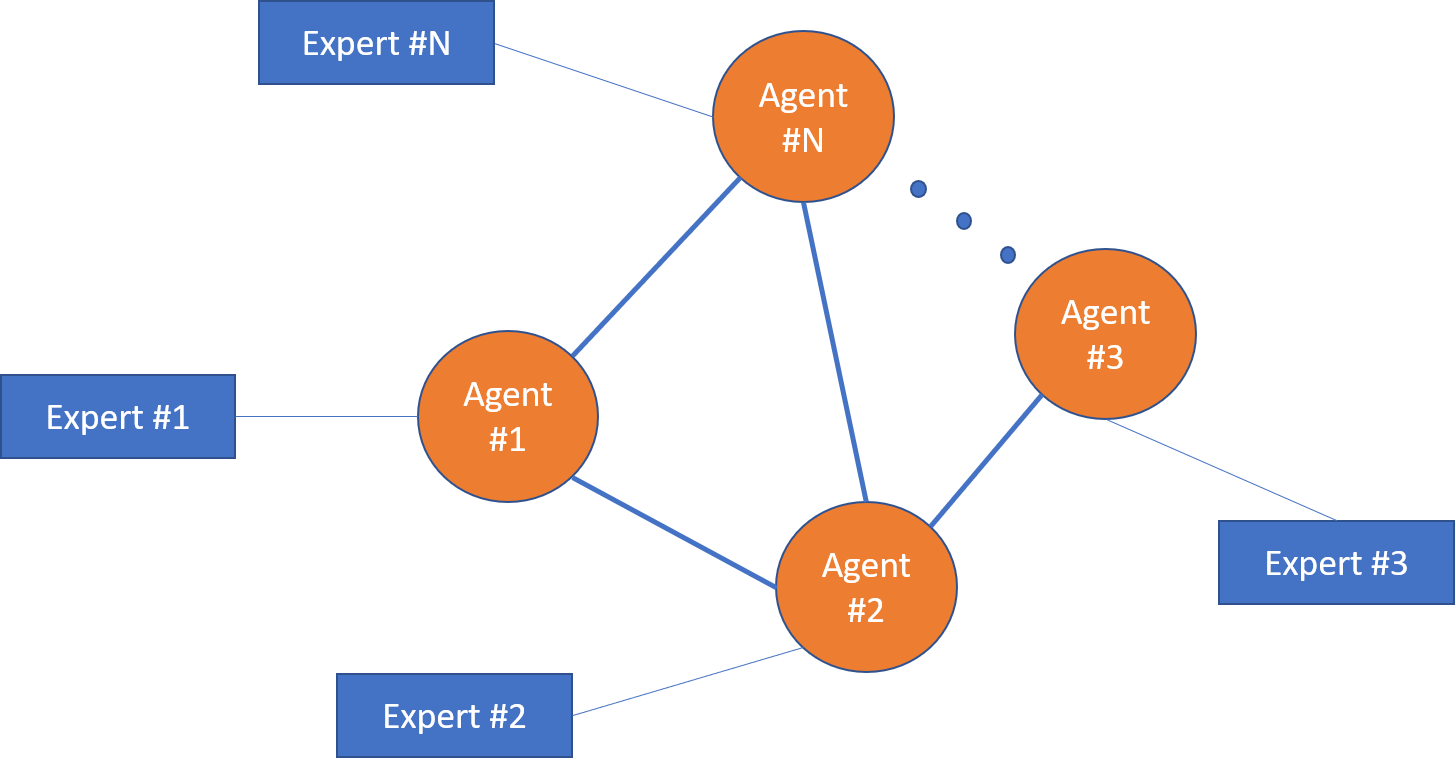}
    \caption{Decentralised Distributed Expert-Assisted Learning (D2EAL)}
    \label{fig:02}
\end{figure}
More formally, it is assumed that the unknown target sequence $y_t \in \mathcal{Y}$, where $\mathcal{Y}$ is called the outcome space, for $t=1,2,\cdots,T$. The $i^{th}$ agent's prediction $\hat{f}_{t,i} \in \mathcal{A}$ and the $i^{th}$ expert's prediction $f_{t,i} \in \mathcal{A}$, where $\mathcal{A}$ is called the action space, $t=1,2,\cdots,T$, and $i=1,2,\cdots,N$. The loss function $l : \mathcal{A} \times \mathcal{Y} \rightarrow [0,1]$, and $l(f,y)$ is convex in its first argument $f \in \mathcal{A}$. Both the outcome space $\mathcal{Y}$ and the action space $\mathcal{A}$ are a convex subset of $\mathbb{R}^2$. 
\begin{definition}
The `best expert' among all of the $N$ expert predictors, with respect to the horizon $T$, is the one which incurs the least cumulative loss in a horizon of $T$. If the ${i^*}^{th}$ expert is the best, then
\begin{equation} \label{equ1}
    i^* = arg\min_{i \in [N]} (\sum_{t=1}^{T} l_{t,i})
\end{equation}
\end{definition}
\begin{definition}
If the best expert is a `true expert', then $\sum_{t=1}^{T} l_{t,i^*} = 0$.
\end{definition}
\begin{definition}
The regret $R_i(T)$ of the $i^{th}$ agent with respect to the best expert is defined as follows:
\begin{equation} \label{equ2}
    R_i(T) = \sum_{t=1}^{T} \hat{l}_{t,i} - \sum_{t=1}^{T} l_{t,i^*}
\end{equation}
\end{definition}
Ideally, $i^{th}$ agent's goal is to keep its cumulative loss $\sum_{t=1}^{T} \hat{l}_{t,i}$ as low as possible compared to the best expert's cumulative prediction loss, or in other words, minimize the regret $R_i(T)$. But, it does not know which agent is assisted by the best expert. Thus, it needs to either estimate the best expert's prediction or possibly form a better prediction by fusing prediction information shared by its neighbours via communication, while keeping the prediction losses as low as possible.
Formally, in the online prediction and learning literature, the ideal objective of a learning agent is to have a regret that is sub-linear in $T$, i.e., $\lim_{T\rightarrow \infty} \frac{R_i(T)}{T} = 0$. 

\subsection{Decentralized Distributed Expert Assisted Learning} 
Based on the distributed learning formulation presented in the previous subsection, we propose the Decentralized Distributed Expert Assisted Learning (D2EAL) algorithm for cooperative target trajectory tracking using heterogeneous MRS, given as Algorithm \ref{alg01}. In D2EAL, $t$ denotes the current discrete-time instant. D2EAL starts by choosing positive integer values for the time horizon $T$, periodic-reset time period $T_o$, and look-ahead discrete time-step window $\tau$, and positive real values for the learning parameters $\eta_{\alpha}$ and $\eta_{w}$. The discrete-time variable $t$, $i^{th}$ agent's prediction, and the weight parameters $\hat{\alpha}_i(t)$, $\hat{\alpha}'_i(t)$, and $\hat{w}_{ij}(t)$ are initialized to $1$, after which the iterative (or loop) process begins, $\forall i \in [N]$. 

An iteration of D2EAL algorithm involves two prediction phases, a communication phase, and a learning phase. In the first prediction phase, the $i^{th}$ agent forms its \lq individual prediction' of $y_{t+1}$ at time $t$, denoted by $\bar{f}_{t+1,i}$, by a weighted convex sum of its prediction of $y_t$, i.e. $\hat{f}_{t,i}$, and its expert's prediction of $y_{t+1}$, i.e. $f_{t+1,i}$, as follows:
\begin{equation} \label{equ2.1}
    \bar{f}_{t+1,i} = \alpha_i(t) f_{t+1,i} + (1 - \alpha_i(t)) \hat{f}_{t,i}
\end{equation}
where
\begin{equation} \label{equ2.2}
    \alpha_i(t) =  \frac{\hat{\alpha}_i(t)}{\hat{\alpha}_i(t) + \hat{\alpha}'_i(t)}
\end{equation}
A similar process is carried out by the $i^{th}$ agent to obtain its \lq individual prediction' of $y_{t+\tau}$ at time $t$, denoted by $\bar{f}_{t+\tau,i}^p$, as follows:
\begin{equation} \label{equ2.3}
    \bar{f}_{t+\tau,i}^p = \alpha_i(t) f_{t+\tau,i}^p + (1 - \alpha_i(t)) \hat{f}_{t+\tau-1,i}^p
\end{equation}
In the communication phase, the $i^{th}$ agent transmits the tuple $\{i, \bar{f}_{t+1,i},\bar{f}_{t+\tau,i}^p,\hat{w}_{ii}(t)\}$ to its neighbouring agents, and in turn, receives $\{j,\bar{f}_{t+1,j},\bar{f}_{t+\tau,j}^p,\hat{w}_{jj}(t)\}$ from its neighbouring agents, $\forall j \in \Omega_i(t)$.

In the second prediction phase, the information obtained by the neighbouring agents is used to form the $i^{th}$ agent's `social prediction' of $y_{t+1}$ at time $t$, i.e. $\hat{f}_{t+1,i}$, and the $i^{th}$ agent's `social prediction' of $y_{t+\tau}$ at time $t$, i.e. $\hat{f}_{t+\tau,i}^p$. $\hat{f}_{t+1,i}$ is obtained by taking a weighted convex sum of all the individual predictions $\bar{f}_{t+1,j}$, $\forall j \in \Lambda_i(t)$, where $\Lambda_i(t) = \Omega_i(t) \cup \{i\}$, as follows:
\begin{equation} \label{equ2.31}
    \hat{f}_{t+1,i} = \sum_{\forall j \in \Lambda_i(t)} w_{ij}(t) \bar{f}_{t+1,j}
\end{equation}
where
\begin{equation} \label{equ2.312}
    w_{ij}(t) =  \frac{\hat{w}_{ij}(t)}{\sum_{\forall j' \in \Lambda_i(t)} \hat{w}_{ij'}(t)}
\end{equation}
and
\begin{equation} \label{equ2.311}
    \hat{w}_{ij}(t) = \left\{
        \begin{array}{ll}
            \hat{w}_{jj}(t) & :\quad \forall j \in \Lambda_i(t) \\
            0 & :\quad otherwise
        \end{array}
    \right.
\end{equation}
Similarly, $\hat{f}_{t+\tau,i}^p$ is obtained by taking a weighted convex sum of all the individual predictions $\bar{f}_{t+\tau,j}^p$, $\forall j \in \Lambda_i(t)$, where $\Lambda_i(t) = \Omega_i(t) \cup \{i\}$, as follows:
\begin{equation} \label{equ2.32}
    \hat{f}_{t+\tau,i}^p = \sum_{\forall j \in \Lambda_i(t)} w_{ij}(t) \bar{f}_{t+\tau,j}^p
\end{equation}
Note that individual predictions are formed by fusing local information available to an agent, whereas, social predictions are formed by fusing the individual predictions of all the neighbors of the agent in the communication network, thus utilizing the information available from the network. 

After $y_{t+1}$ is revealed/observed at time $t+1$, the first learning phase is executed in which the weights $\hat{\alpha}_i(t)$, $\hat{\alpha}'_i(t)$, $\hat{w}_{ii}$ are updated using the exponential weights scheme as follows: 
\begin{subequations} \label{equ2.4}
\begin{align}
    \hat{\alpha}_i(t+1) &= \hat{\alpha}_i(t) \exp{(-\eta_{\alpha} l_{t+1,i})} \\
  \hat{\alpha}'_i(t+1) &= \hat{\alpha}'_i(t) \exp{(-\eta_{\alpha} \hat{l}_{t+1,i}^{-})} \\
  \hat{w}_{ii}(t+1) &= \hat{w}_{ii}(t) \exp{(-\eta_{w} \bar{l}_{t+1,i})}
\end{align}
\end{subequations} 
where $l_{t+1,i} = l(f_{t+1,i},y_{t+1})$ is the loss incurred by the $i^{th}$ expert's prediction of $y_{t+1}$, $\hat{l}_{t+1,i}^{-} = l(\hat{f}_{t,i},y_{t+1})$ is the loss incurred by the $i^{th}$ agent's prediction of $y_t$ compared to $y_{t+1}$, and $\bar{l}_{t,i} = l(\bar{f}_{t+1,i},y_{t+1})$ is the loss incurred by the $i^{th}$ agent's individual prediction of $y_{t+1}$.   

A periodic reset of the weights happens after every $T_o$ discrete time steps -- weights are re-initialized to 1 to remove any potential biases that might have accumulated over the period of $T_o$ discrete steps. Further, the weights are normalized using a decentralized normalization scheme, briefly described as follows:
given the machine's least precision $\delta$, for the $i^{th}$ robot at discrete time $t$, $t=1,2,\cdots,T$, $\forall i \in [N]$:
\begin{itemize}
    \item if $\hat{w}_{ii}(t) \leq \delta$: \\
        $\hat{w}_{ii}(t) \leftarrow \hat{w}_{ii}(t)/\delta$ \\
        $\quad nrmcnt_{ii} \leftarrow nrmcnt_{ii} + 1$
    \item send $\{i, \hat{w}_{ii}(t), nrmcnt_{ii}\}$ to and receive $\{j, \hat{w}_{jj}(t), nrmcnt_{jj}\}$ from neighbours $j\in \Omega_i(t)$
    \item $\hat{w}_{ij}(t) = \left\{
        \begin{array}{ll}
            \hat{w}_{jj}(t) & :\quad \forall j \in \Lambda_i(t) \\
            0 & :\quad otherwise
        \end{array}
    \right.$
    \item $\forall j \in \Lambda_i(t)$: \\
    if $nrmcnt_{jj} > \min_{j' \in \Lambda_i(t)}nrmcnt_{j'j'}$: $\hat{w}_{ij}(t) \leftarrow 0$
\end{itemize}
A similar procedure is used for normalization of the weights $\hat{\alpha}_i(t)$ and $\hat{\alpha}'_i(t)$ as well.

Note that $f_{t+1,i}$, $\bar{f}_{t+1,i}$, and $\hat{f}_{t+1,i}$ are one-step look-ahead predictions of $y_{t}$, which are involved in the weights-update process. Whereas, $f_{t+\tau,i}^p$, $\bar{f}_{t+\tau,i}^p$, and $\hat{f}_{t+\tau,i}^p$ are $\tau$-step look-ahead predictions of $y_{t}$, which are not involved in the weights-update process. Also note that $\hat{f}_{t+\tau,i}^p \equiv \hat{x}_{(t+\tau|t),T_g}^{i}$ is used by the $i^{th}$ robot in its control law specified by equations (\ref{eqt02.2}) and (\ref{eqt02.3}). 
D2EAL is summarized in Algorithm \ref{alg01}.
\begin{algorithm} 
 \caption{: D2EAL algorithm for the $i^{th}$ agent, $\forall i \in [N]$}
 \begin{algorithmic}[1] \label{alg01}
 \renewcommand{\algorithmicrequire}{\textbf{Initialization:}}
 \renewcommand{\algorithmicensure}{\textbf{Choose:}}
 \ENSURE $T,T_o,\tau \geq 1$ (integers); $\eta_{\alpha}, \eta_w > 0$
 \REQUIRE $\hat{w}_{ii}(0) = 1$, $\hat{\alpha}_i(0) = 1$, $\hat{\alpha}'_i(0) = 1$ \\ 
  $\quad \quad \quad \quad \hat{f}_{0,i} = f_{1,i}$, $t = 0$ \\ 
  \WHILE{$t \leq T$} 
  \IF {($t > 0$)}
  \STATE Observe $y_t$
  \STATE $l_{t,i} = l(f_{t,i},y_t)$, $\hat{l}_{t,i}^{-} = l(\hat{f}_{t-1,i},y_t)$
  \STATE $\hat{\alpha}_i(t) = \hat{\alpha}_i(t-1) \exp{(-\eta_{\alpha} l_{t,i})}$
  \STATE $\hat{\alpha}'_i(t) = \hat{\alpha}'_i(t-1) \exp{(-\eta_{\alpha} \hat{l}_{t,i}^{-})}$
  \STATE $\bar{l}_{t,i} = l(\bar{f}_{t,i},y_t)$
  \STATE $\hat{w}_{ii}(t) = \hat{w}_{ii}(t-1) \exp{(-\eta_{w} \bar{l}_{t,i})}$
  \ENDIF
  \STATE \textbf{Periodic Reset}: re-initialize the weights $\hat{\alpha}_i(t)$, $\hat{\alpha}'_i(t)$, and $\hat{w}_{ii}(t)$ to $1$ after every $T_o$ discrete time steps
  \STATE $\alpha_i(t) =  \frac{\hat{\alpha}_i(t)}{\hat{\alpha}_i(t) + \hat{\alpha}'_i(t)}$
  \STATE access $i^{th}$ expert's one-step look-ahead prediction of $y_{t}$ as $f_{t+1,i}$
  \STATE $\bar{f}_{t+1,i} = \alpha_i(t) f_{t+1,i} + (1 - \alpha_i(t)) \hat{f}_{t,i}$
  \STATE access $i^{th}$ expert's $\tau$-step look-ahead prediction of $y_{t}$ as $f_{t+\tau,i}^p$
  \STATE $\bar{f}_{t+\tau,i}^p = \alpha_i(t) f_{t+\tau,i}^p + (1 - \alpha_i(t)) \hat{f}_{t+\tau-1,i}^p$
  \STATE transmit $\{i, \bar{f}_{t+1,i},\bar{f}_{t+\tau,i}^p,\hat{w}_{ii}(t)\}$ to the neighbouring agents, and receive $\{j,\bar{f}_{t+1,j},\bar{f}_{t+\tau,j}^p,\hat{w}_{jj}(t)\}$ from the neighbouring agents, $\forall j \in \Omega_i(t)$
  \STATE $\Lambda_i(t) = \Omega_i(t) \cup \{i\}$
  \STATE $\hat{w}_{ij}(t) = \left\{
        \begin{array}{ll}
            \hat{w}_{jj}(t) & :\quad \forall j \in \Lambda_i(t) \\
            0 & :\quad otherwise
        \end{array}
    \right.$
  \STATE $w_{ij}(t) =  \frac{\hat{w}_{ij}(t)}{\sum_{\forall j' \in \Lambda_i(t)} \hat{w}_{ij'}(t)}$
  \STATE $\hat{f}_{t+1,i} = \sum_{\forall j \in \Lambda_i(t)} w_{ij}(t) \bar{f}_{t+1,j}$
  \STATE Assign $\hat{f}_{t+1,i}$ as the $i^{th}$ agent's one-step look-ahead prediction of the target sequence $y_t$
  \STATE $\hat{f}_{t+\tau,i}^p = \sum_{\forall j \in \Lambda_i(t)} w_{ij}(t) \bar{f}_{t+\tau,j}^p$
  \STATE Assign $\hat{f}_{t+\tau,i}^p$ as the $i^{th}$ agent's $\tau$-step look-ahead prediction of the target sequence $y_t$ 
  \STATE $t = t + 1$
  \ENDWHILE
 \end{algorithmic} 
\end{algorithm}

\section{Theoretical Analysis of D2EAL}
In this section, we present theoretical analysis of the D2EAL algorithm (without periodic reset) in terms of the $i^{th}$ agent's (or robot's) regret performance. In the following subsections, we define various regret measures and derive their worst-case bounds, which are then minimized with respect to the learning parameters $\eta_{\alpha}$ and $\eta_w$ to give the optimal worst-case regret bounds.  

The $i^{th}$ agent incurs a loss $\hat{l}_{t,i} := l(\hat{f}_{t,i},y_t)$, where $l(\cdot,\cdot) \in [0,1]$ is a convex loss function. The loss due to $i^{th}$ expert's prediction at time $t$ is defined as $l_{t,i} := l(f_{t,i},y_t)$. The neighbour set is defined as: $\Omega_i(t) = \{j : j^{th}$ agent is the neighbour of $i^{th}$ agent at time $t$, as per $G(t)\}$, where $i=1,2,\cdots,N$. Further, define $\Lambda_i(t) := \Omega_i(t) \cup \{i\}$.

It is assumed that the unknown target sequence $y_t \in \mathcal{Y}$, where $\mathcal{Y}$ is called the outcome space, for $t=1,2,\cdots,T$. The $i^{th}$ agent's prediction $\hat{f}_{t,i} \in \mathcal{A}$ and the $i^{th}$ expert's prediction $f_{t,i} \in \mathcal{A}$, where $\mathcal{A}$ is called the action space, $t=1,2,\cdots,T$, and $i=1,2,\cdots,N$. The loss function $l : \mathcal{A} \times \mathcal{Y} \rightarrow [0,1]$, and $l(p,y)$ is convex in its first argument $p \in \mathcal{A}$. Both the outcome $\mathcal{Y}$ and the action $\mathcal{A}$ are a convex subset of $\mathbb{R}^n$.

At time $t$, $\hat{l}_{t,i}^{-} = l(\hat{f}_{t-1,i},y_t)$ is the loss incurred by $i^{th}$ agent's previous time-step prediction $\hat{f}_{t-1,i}$, with respect to $y_{t}$. Define $\hat{L}_{t,i}^{-} := \sum_{s=1}^{t} \hat{l}_{t,i}^{-}$. 

$\bar{f}_{t+1,j}$ denotes the $i^{th}$ agent's `individual prediction' of $y_{t+1}$ at time $t$. Loss incurred by the $i^{th}$ agent's `individual prediction' with respect to $y_t$ is given as $\bar{l}_{t,i} = l(\bar{f}_{t,i},y_t)$. Define $\bar{L}_{t,i} := \sum_{s=1}^{t} \bar{l}_{t,i}$. 

$\hat{f}_{t+1,i}$ is called as $i^{th}$ agent's `social prediction' of $y_{t+1}$ at time $t$. Loss incurred by the $i^{th}$ agent's `social prediction' with respect to $y_t$ is given as $\hat{l}_{t,i} = l(\hat{f}_{t,i},y_t)$. Define $\hat{L}_{t,i} := \sum_{s=1}^{t} \hat{l}_{t,i}$.

\subsection{Agent's Individual Prediction Regret Analysis}
This subsection presents theoretical results on $i^{th}$ agent's individual prediction regret. $i^{th}$ agent's individual prediction regret refers to the regret for its individual prediction $\bar{f}_{t,i}$ with respect to the information available to the agent by its own expert's prediction $f_{t,i}$ and its previous time-step prediction $\hat{f}_{t-1,i}$.     
\begin{lemma} \label{lem01}
Using D2EAL algorithm for predicting an unknown signal $y_t$ with dynamics of any arbitrary structure, with the time horizon $T \geq 1$, and the learning parameters $\eta_{\alpha} > 0$ and $\eta_w > 0$, the D2EAL algorithm satisfies the following:
\begin{equation} \label{equat07}
    R_i^I(T) := \bar{L}_{T,i} - \min\left\{ L_{T,i}, \hat{L}_{T,i}^{-} \right\} \leq \frac{\eta_{\alpha} T}{8} + \frac{\log2}{\eta_{\alpha}}
\end{equation}
where $R_i^I(T)$ is defined as the $i^{th}$ agent's individual prediction's regret over a horizon of $T$, $\forall i \in [N]$.
\end{lemma}
\begin{proof}
Consider the potential function: $\phi_i(t) = \frac{1}{\eta_{\alpha}} \log(\hat{\alpha}_i(t) + \hat{\alpha}'_i(t))$. Since $\hat{\alpha}_i(0) = \hat{\alpha}'_i(0) = 1$, we have $\hat{\alpha}_i(t) = \exp{\left(-\eta_{\alpha} \sum_{s=1}^{t} l_{s,i} \right)} = \exp{\left(-\eta_{\alpha} L_{t,i}\right)}$, and $\hat{\alpha}'_i(t) = \exp{\left(-\eta_{\alpha} \sum_{s=1}^{t} \hat{l}_{s,i}^{-} \right)} = \exp{\left(-\eta_{\alpha} \hat{L}_{t,i}^{-}\right)}$. This implies
\begin{equation} \label{equat01}
    \phi_i(t) = \frac{1}{\eta_{\alpha}} \log \left( \exp{\left(-\eta_{\alpha} L_{t,i}\right)} + \exp{\left(-\eta_{\alpha} \hat{L}_{t,i}^{-}\right)} \right)
\end{equation}
and
\begin{equation} \label{equat02}
    \phi_i(0) = \frac{1}{\eta_{\alpha}} \log2
\end{equation}
Therefore
\begin{equation} \label{equat03}
    \phi_i(t) - \phi_i(0) = \frac{1}{\eta_{\alpha}} \log \left( \frac{\exp{\left(-\eta_{\alpha} L_{t,i}\right)} + \exp{\left(-\eta_{\alpha} \hat{L}_{t,i}^{-}\right)}}{2} \right)
\end{equation}
This further implies
\begin{equation} \label{equat04a}
    \phi_i(T) - \phi_i(0) \geq -\min\left\{L_{T,i},\hat{L}_{T,i}^{-}\right\} - \frac{\log2}{\eta_{\alpha}} 
\end{equation}
Now, consider the per-step decrease in the potential function $\phi_i(t)$ as follows:
\begin{equation} \label{equat05}
    \phi_i(t) - \phi_i(t-1) = \frac{1}{\eta_{\alpha}} \log \left( \frac{\hat{\alpha}_i(t) + \hat{\alpha}'_i(t)}{\hat{\alpha}_i(t-1) + \hat{\alpha}'_i(t-1)} \right)
\end{equation}
Further simplification leads to
\begin{equation}
    \begin{array}{cc}
         \phi_i(t) - \phi_i(t-1) =& \frac{1}{\eta_{\alpha}} \log ( {\alpha}_i(t-1) \exp{\left( -\eta_{\alpha} l_{t,i} \right)}
   \\&+ ({\alpha'}_i(t-1)) \exp{\left( -\eta_{\alpha} \hat{l}_{t,i}^{-} \right)} )
    \end{array}
\end{equation}
First, using Hoeffding's Lemma (ch.2,\cite{massart2007concentration}) and then, using Jensen's Inequality (\cite{jensen1906convex}), we get the following:
\begin{equation} \label{equat06}
    \phi_i(t) - \phi_i(t-1) \leq -\bar{l}_{t,i} + \frac{\eta_{\alpha}}{8}
\end{equation}
Unrolling the above equation for $t=1,2,...,T$ to get $T$ equations, and adding up all those equations leads to the following:
\begin{equation} \label{equat07.1}
    \phi_i(T) - \phi_i(0) \leq -\bar{L}_{T,i} + \frac{\eta_{\alpha} T}{8}
\end{equation}
Equations (\ref{equat04a}) and (\ref{equat07.1}) lead to the desired result given as equation (\ref{equat07}).  
\end{proof} 
\begin{corollary}
Minimizing the worst-case regret bound with respect to $\eta_{\alpha}$ given in equation (\ref{equat07}) leads to the optimal learning parameter choice $\eta_{\alpha} = \sqrt{\frac{8\log2}{T}}$. This gives the following sub-linear worst-case regret bound:
\begin{equation} \label{equat08}
    R_i^I(T) \leq \sqrt{\frac{T}{2} \log2}
\end{equation}
\end{corollary}

\subsection{Agent's Social Prediction Regret Analysis}
This subsection presents theoretical results on $i^{th}$ agent's social prediction regret. $i^{th}$ agent's social prediction regret refers to the regret for its social prediction $\hat{f}_{t,i}$ with respect to the information available to the agent by its neighbouring agent's individual predictions $\bar{f}_{t,j}$ ($\forall j \in \Omega_i(t)$) and its own individual prediction $\bar{f}_{t,i}$.

\textit{Assumption 1:} $\Omega_i(t) \subseteq \Omega_i(t-1)$, i.e., the set of neighbouring agents of the $i^{th}$ agent at time $t$ is either a subset of or equal to the set of neighbouring agents of the $i^{th}$ agent at time $t-1$.

Note that Assumption 1 implies: $\Lambda_i(t) \subseteq \Lambda_i(t-1)$, and $d_i(t) \leq d_i(t-1)$, where $d_i(t) = \sum_{\forall j \in \Lambda_i(t)} (1) = \sum_{j=1}^{N} \mathbf{1}(j \in \Lambda_i(t))$ is the degree of the $i^{th}$ agent (node) at time $t$, and $\mathbf{1}(.)$ is the indicator function. 
\begin{lemma} \label{lem02}
Using D2EAL under Assumption 1, for an unknown signal $y_t$ with dynamics of any arbitrary structure, with the time horizon $T \geq 1$, and the learning parameters $\eta_{\alpha} > 0$ and $\eta_w > 0$, the D2EAL algorithm satisfies the following:
\begin{equation} \label{equat17}
    R_i^S(T) := \hat{L}_{T,i} -\min_{j \in \Lambda_i(T)} \left\{ \bar{L}_{T,j} \right\} \leq \frac{\eta_w T}{8} + \frac{\log d_i(0)}{\eta_w}
\end{equation}
where $R_i^S(T)$ is defined as the $i^{th}$ agent's social prediction's regret over a horizon of $T$, $\forall i \in [N]$.
\end{lemma}
\begin{proof}
Consider the potential function: $\Phi_i(t) = \frac{1}{\eta_w} \log\left( \sum_{\forall j \in \Lambda_i(t)} \hat{w}_{ij}(t) \right)$. Since $\hat{w}_{ij}(0) = 1$ for $j \in \Lambda_i(0)$, and $\hat{w}_{ij}(0) = 0$ otherwise, we have $\hat{w}_{ij}(t) = \exp{\left( -\eta_w \sum_{s=1}^{t} \bar{l}_{s,j} \right)} = \exp{\left( -\eta_w \bar{L}_{t,j} \right)}$ for $j \in \Lambda_i(t)$, and $\hat{w}_{ij}(t) = 0$ otherwise. Note that $\hat{w}_{ij}(t) = \hat{w}_{jj}(t) = \exp{\left( -\eta_w \sum_{s=1}^{t} \bar{l}_{s,j} \right)} = \exp{\left( -\eta_w \bar{L}_{t,j} \right)}$ if $j \in \Lambda_i(t)$, $\forall i \in [N]$. This implies
\begin{equation} \label{equat09}
    \Phi_i(t) = \frac{1}{\eta_w} \log\left( \sum_{\forall j \in \Lambda_i(t)} \exp{\left( -\eta_w \bar{L}_{t,j} \right)}  \right)
\end{equation}
and
\begin{equation} \label{equat10}
    \Phi_i(0) = \frac{1}{\eta_w} \log\left( d_i(0) \right)
\end{equation}
Thus, we have
\begin{equation} \label{equat11}
    \Phi_i(T) - \Phi_i(0) = \frac{1}{\eta_w} \log\left( \frac{\sum_{\forall j \in \Lambda_i(T)} \exp{\left( -\eta_w \bar{L}_{T,j} \right)}}{d_i(0)} \right)
\end{equation}
This further implies
\begin{equation} \label{equat12}
    \Phi_i(T) - \Phi_i(0) \geq -\min_{j \in \Lambda_i(T)} \left\{ \bar{L}_{T,j} \right\} - \frac{\log d_i(0)}{\eta_w}
\end{equation}
Per-step decrease in the potential function $\Phi_i(t)$ can be given as follows:
\begin{equation} \label{equat13}
    \Phi_i(t) - \Phi_i(t-1) = \frac{1}{\eta_w} \log\left( \frac{\sum_{\forall j \in \Lambda_i(t)} \hat{w}_{ij}(t)}{\sum_{\forall j \in \Lambda_i(t-1)} \hat{w}_{ij}(t-1)} \right)
\end{equation}
or
\begin{equation} \label{equat13.1}
    \Phi_i(t) - \Phi_i(t-1) = \frac{ \log\left( \frac{\sum_{\forall j \in \Lambda_i(t)} \exp{\left( -\eta_w \bar{L}_{t,j} \right)}}{\sum_{\forall j \in \Lambda_i(t-1)} \exp{\left( -\eta_w \bar{L}_{t-1,j} \right)}} \right)}{\eta_w}
\end{equation}

Assumption 1 implies
\begin{equation} \label{equat13.2}
    \Phi_i(t) - \Phi_i(t-1) \leq \frac{ \log\left( \frac{\sum_{\forall j \in \Lambda_i(t-1)} \exp{\left( -\eta_w \bar{L}_{t,j} \right)}}{\sum_{\forall j \in \Lambda_i(t-1)} \exp{\left( -\eta_w \bar{L}_{t-1,j} \right)}} \right)}{\eta_w}
\end{equation}
Further simplification leads to 
\begin{equation} \label{equat14}
\begin{array}{cc}
     \Phi_i(t) - \Phi_i(t-1) \leq \\\frac{1}{\eta_w} \log\left( {\sum_{\forall j \in \Lambda_i(t-1)} {w}_{ij}(t-1) \exp{\left( -\eta_w \bar{l}_{t,j} \right)}} \right)
\end{array}
\end{equation}
First, using Hoeffding's Lemma (ch.2,\cite{massart2007concentration}) and then, using Jensen's Inequality (\cite{jensen1906convex}), we get the following:
\begin{equation} \label{equat15}
    \Phi_i(t) - \Phi_i(t-1) \leq -\hat{l}_{t,i} + \frac{\eta_w}{8}
\end{equation}
Unrolling the above equation for $t=1,2,...,T$ to get $T$ equations, and adding up all those equations leads to the following:
\begin{equation} \label{equat16}
    \Phi_i(T) - \Phi_i(0) \leq -\hat{L}_{T,i} + \frac{\eta_w T}{8}
\end{equation}
Equations (\ref{equat12}) and (\ref{equat16}) lead to the desired result given as equation (\ref{equat17}).
\end{proof}
\begin{corollary}
Minimizing the worst-case regret bound with respect to $\eta_w$ given in equation (\ref{equat17}) leads to the optimal learning parameter choice $\eta_w = \sqrt{\frac{8\log d_i(0)}{T}}$. This gives the following sub-linear worst-case regret bound:
\begin{equation} \label{equat18}
    R_i^S(T) \leq \sqrt{\frac{T}{2} \log d_i(0)}
\end{equation}
\end{corollary}

\subsection{Global Regret Analysis}
This subsection presents theoretical results on $i^{th}$ agent's individual prediction global regret and social prediction global regret. $i^{th}$ agent's individual prediction global regret refers to the regret for its individual prediction $\bar{f}_{t,i}$ with respect to the best expert's prediction $f_{t,i^*}$ in the network, where $i^* = \arg \min_{j\in [N]} l_{T,j}$. $i^{th}$ agent's social prediction global regret refers to the regret for its social prediction $\hat{f}_{t,i}$ with respect to the best individual prediction present in the network $\bar{f}_{t,j^*}$, where $j^* = \arg \min_{j\in[N]} \bar{L}_{T,j}$. 

\textit{Assumption 2:} $|l(x_1,y) - l(x_2,y)| \leq L_1 ||x_1 - x_2||$ and $|l(x,y_1) - l(x,y_2)| \leq L_2 ||y_1 - y_2||$ , where $L_1$ and $L_2$ are Lipschitz constants, and $||\cdot||$ is the Euclidean norm. 

\textit{Assumption 3:} $||f_{t,i} - f_{t,j}|| \leq \delta_{ij}(t) \leq \delta_t$ , $\forall i, j \in [N]$, $t = 1,2,\cdots,T$, where $||\cdot||$ is the Euclidean norm.

\subsubsection{Agent's Individual Prediction Global Regret Analysis}
\begin{theorem} \label{thm01}
Using D2EAL under Assumptions 2 and 3, for an unknown signal $y_t$ with dynamics of any arbitrary structure, with the time horizon $T \geq 1$, and the learning parameters $\eta_{\alpha} > 0$ and $\eta_w > 0$, $\exists \Delta_o \geq \sum_{t=1}^{T} \delta_t$ such that the D2EAL algorithm satisfies the following:
\begin{equation} \label{equat49.1}
     R_i^{GI}(T) := \bar{L}_{T,i} - L_{T,i^*} \leq \frac{\eta_{\alpha} T}{8} + \frac{\log2}{\eta_{\alpha}} + L_1 \Delta_o
\end{equation}
where $R_i^{GI}(T)$ is the $i^{th}$ agent's individual prediction global regret, and $i^* = \arg\min_{j \in [N]} L_{T,j}$, $\forall i \in [N]$.
\end{theorem}
\begin{proof}
Using Lemma \ref{lem01}, consider equation (\ref{equat07}), and note that the following holds true:
\begin{equation} \label{equat45}
    \bar{L}_{T,i} - L_{T,i} \leq \frac{\eta_{\alpha} T}{8} + \frac{\log2}{\eta_{\alpha}}
\end{equation}
The above inequality can be re-written as follows:
\begin{equation} \label{equat46}
\begin{array}{cc}
    R_i^{GI}(T) = \bar{L}_{T,i} - L_{T,i^*} \leq& \frac{\eta_{\alpha} T}{8} + \frac{\log2}{\eta_{\alpha}} \\ &+ (L_{T,i} - L_{T,i^*})
\end{array}
\end{equation}
Assumption 2 leads to the following:
\begin{equation} \label{equat47}
    |l_{t,i} - l_{t,i^*}| \leq L_1 ||f_{t,i} - f_{t,i^*}||
\end{equation}
Further using assumption 3, we get:
\begin{equation} \label{equat48}
    l_{t,i} - l_{t,i^*} \leq L_1 \delta_{ii^*}(t) \leq L_1 \delta_t
\end{equation}
This implies
\begin{equation} \label{equat49}
    L_{t,i} - L_{t,i^*} \leq L_1 \sum_{t=1}^{T} \delta_{ii^*}(t) \leq L_1 \sum_{t=1}^{T} \delta_t 
\end{equation}
Since $\sum_{t=1}^{T} \delta_t \leq \Delta_o$, we get the desired result given as equation \ref{equat49.1}.
\end{proof}
\subsubsection{Agent's Social Prediction Regret w.r.t. Best Expert}
Using Lemma \ref{lem02}, consider equation (\ref{equat17}), and note that the following holds true for some $j \in \Lambda_i(T)$:
\begin{equation} \label{equat50}
    R_i^S(T) = \hat{L}_{T,i} - \bar{L}_{T,j} \leq \frac{\eta_w T}{8} + \frac{\log d_i(0)}{\eta_w}
\end{equation}
\begin{corollary} \label{corr03}
Under assumptions 1, 2, and 3, adding inequalities (\ref{equat50}) (with $j=i$) and (\ref{equat49.1}) (using Theorem \ref{thm01}), we can show that the D2EAL algorithm satisfies the following:
\begin{equation} \label{equat55}
\begin{array}{cc}
    R_i^{BE}(T) := \hat{L}_{T,i} - {L}_{T,i^*} \leq& \frac{\eta_w T}{8} + \frac{\log d_i(0)}{\eta_w} + \frac{\eta_{\alpha} T}{8} \\ &+ \frac{\log2}{\eta_{\alpha}} + L_1 \Delta_o 
\end{array}
\end{equation}
where $R_i^{BE}(T)$ is the $i^{th}$ agent's social prediction regret with respect to the best expert $i^*$ for a horizon of $T$, $\forall i \in [N]$. 
\end{corollary}

\subsubsection{Agent's Social Prediction Global Regret Analysis}

Note that if ${i^{*}}^{th}$ expert is the best expert for the horizon $T$, then its cumulative loss satisfies 
\begin{equation} \label{equat56}
    L_{T,i^*} \leq cT
\end{equation}
where $c \in [0,1]$, since the convex loss function $l(p,y) \in [0,1]$. 

\textit{Assumption 4:} The best expert is \textit{sub-linear} with respect to horizon $T$, i.e.,
\begin{equation} \label{57}
    L_{T,i^*} \leq c_0 T^{1-\alpha} \leq cT
\end{equation}
such that $c_0 \leq cT^{\alpha}$, where $\alpha \in (0,1]$.
\begin{theorem}
Under Assumptions 1, 2, 3, and 4, for an unknown signal $y_t$ with dynamics of any arbitrary structure, with the time horizon $T \geq 1$, and the learning parameters $\eta_{\alpha} > 0$ and $\eta_w > 0$, $\exists \Delta_o \geq \sum_{t=1}^{T} \delta_t$ such that the D2EAL algorithm satisfies the following:
\begin{equation} \label{equat58}
\begin{array}{ll}
     R_i^{GS}(T) := \hat{L}_{t,i} - \bar{L}_{t,j^*} \leq& \frac{\eta_w T}{8} + \frac{\log d_i(0)}{\eta_w} + \frac{\eta_{\alpha} T}{8}\\ &+ \frac{\log2}{\eta_{\alpha}} + L_1 \Delta_o + c_0 T^{1-\alpha}
\end{array}
\end{equation}
where $R_i^{GS}(T)$ is the $i^{th}$ agent's social prediction global regret, and $j^* = \arg\min_{j \in [N]} \bar{L}_{T,j}$, $\forall i \in [N]$.
\end{theorem}
\begin{proof}
Using Corollary \ref{corr03} and assumption 4, we get the desired result.
\end{proof}

\subsection{Convergence Analysis}
Consider $0 < \hat{w}_{ii}(0) \leq 1$, where $i = 1,2,\cdots,N$. 
For $t=1,2,\cdots, \bar{T}$, $\hat{w}_{ij}(t) = \hat{w}_{jj}(0) \exp\{-\eta_w \bar{L}_{t,j}\}$, $\forall j \in \Lambda_i(t)$. 
Further, define $j'_{*}(t) := \arg \min_{j'\in \Lambda_i(t)} \bar{L}_{t,j'}$, i.e., $j'_{*}(t)$ is the index of the robot which incurs the least cumulative loss among all other robots in the index set $\Lambda_i(t) = \Omega_i(t) \cup \{i\}$ at time $t$, where $\Omega_i(t)$ is the neighbours' index set of the $i^{th}$ robot at time $t$. 

Consider the weight $w_{ij}(t)$, which can be re-written as:
\begin{equation} \label{equat60}
    w_{ij}(t) =  \frac{\hat{w}_{ij}(0) \exp\{ -\eta_w (\bar{L}_{t,j}-\bar{L}_{t,j'_{*}(t)})\}}{\sum_{j'\in \Lambda_i(t)} \hat{w}_{ij'}(0) \exp\{ -\eta_w (\bar{L}_{t,j'}-\bar{L}_{t,j'_{*}(t)})\}}
\end{equation}
or,
\begin{equation} \label{equat61}
\begin{array}{cc}
    & w_{ij}(t) = \\  &\frac{\hat{w}_{ij}(0) \exp\{ -\eta_w (\bar{L}_{t,j}-\bar{L}_{t,j'_{*}(t)})\}}{\hat{w}_{ij'_{*}(t)}(0) + \sum_{j'\in \Lambda_i(t) \setminus \{j'_{*}(t)\}} \hat{w}_{ij'}(0) \exp\{ -\eta_w (\bar{L}_{t,j'}-\bar{L}_{t,j'_{*}(t)})\}}
\end{array}
\end{equation}

\textit{Assumption 5:} Cumulative loss for the $j^{th}$ agent, $\forall j \in \Lambda_i(t) \setminus \{j'_{*}(t)\}$, satisfies
\begin{equation} \label{equat62}
    t \geq \bar{\epsilon} t \geq \bar{L}_{t,j} - \bar{L}_{t,j'_{*}(t)} \geq \ubar{\epsilon} t > 0
\end{equation}
where $1 \geq \bar{\epsilon} \geq \ubar{\epsilon} > 0$. 

\textit{Assumption 6:} Both $\lim_{t \rightarrow \infty} j'_{*}(t)$ and $\lim_{t \rightarrow \infty} \Lambda_i(t)$ exist uniquely.

\begin{theorem}
Under the assumptions 5 and 6, D2EAL algorithm's weights $w_{ij}(t)$ satisfy the following:
\begin{equation} \label{equat63}
    \lim_{t \rightarrow \infty} w_{ij}(t) = 0, \quad \forall j \in \lim_{t \rightarrow \infty} \Lambda_i(t) \setminus \{j'_{*}(t)\}
\end{equation}
and
\begin{equation} \label{equat64}
    \lim_{t \rightarrow \infty} w_{ij'_{*}(t)}(t) = 1
\end{equation}
where $j'_{*}(t)$ is the index of the neighbor of the $i^{th}$ robot whose individual prediction incurs the least cumulative loss at time $t$, i.e., $j'_{*}(t) = \arg \min_{j'\in \Lambda_i(t)} \bar{L}_{t,j'}$, $\forall i \in [N]$.
\end{theorem}
\begin{proof}
Using assumption 5 on equation (\ref{equat61}), we get
\begin{equation} \label{equat65}
\begin{array}{cc}
    &\frac{\hat{w}_{ij}(0) \exp\{ -\eta_w \bar{\epsilon} t\}}{\hat{w}_{ij'_{*}(t)}(0) + \sum_{j'\in \Lambda_i(t) \setminus \{j'_{*}(t)\}} \hat{w}_{ij'}(0) \exp\{ -\eta_w \ubar{\epsilon} t\}} \\ &\leq w_{ij}(t) \leq \\  &\frac{\hat{w}_{ij}(0) \exp\{ -\eta_w \ubar{\epsilon} t\}}{\hat{w}_{ij'_{*}(t)}(0) + \sum_{j'\in \Lambda_i(t) \setminus \{j'_{*}(t)\}} \hat{w}_{ij'}(0) \exp\{ -\eta_w \bar{\epsilon} t\}}
\end{array}
\end{equation}
for $\forall j \in \Lambda_i(t) \setminus \{j'_{*}(t)\}$, and 
\begin{equation} \label{equat66}
\begin{array}{cc}
    &\frac{\hat{w}_{ij'_{*}(t)}(0)}{\hat{w}_{ij'_{*}(t)}(0) + \sum_{j'\in \Lambda_i(t) \setminus \{j'_{*}(t)\}} \hat{w}_{ij'}(0) \exp\{ -\eta_w \ubar{\epsilon} t\}} \\ &\leq w_{ij'_{*}(t)}(t) \leq \\  &\frac{\hat{w}_{ij'_{*}(t)}(0)}{\hat{w}_{ij'_{*}(t)}(0) + \sum_{j'\in \Lambda_i(t) \setminus \{j'_{*}(t)\}} \hat{w}_{ij'}(0) \exp\{ -\eta_w \bar{\epsilon} t\}}
\end{array}
\end{equation}
Now, taking $\lim_{t \rightarrow \infty}(\cdot)$ on equations (\ref{equat65}) and (\ref{equat66}), under assumption 6, leads to the desired result. 
\end{proof}
Similar to convergence analysis of weights $w_{ij}(t)$, one can also derive convergence results for weights $\alpha_i(t)$ as well.

\section{Performance Evaluation}
In this section, two simulation studies are presented. The first simulation study involves performance evaluation for $N=6$ robots, with more emphasis on how D2EAL handles adverse dynamic biases or drift in the predictions. Whereas, the second simulation study evaluates how scalable D2EAL is, as the no. of robots $N$ are increased in the MRS. 

In the first study, D2EAL is evaluated using a simulated environment with $N=6$ robots communicating over a dynamic network while performing the task of cooperative target tracking as discussed in section II, for a horizon of $T=1400$ discrete time steps and a sampling period of $\Delta T = 0.1$ second, with a discrete time-step look-ahead window of $\tau = 1$. The communication network is considered to be a random undirected graph with a link drop probability of $0.1$, whose underlying base graph is an undirected connected linear graph. Note that this choice of the base graph corresponds to the worst case for network connectivity among other undirected connected topologies. With this random graph setup for communication network, Fig.\ref{fig:03} shows the percentage frequency of link-drops that occur in the communication network for one of the simulation runs.   
\begin{figure}[h]
    \centering
    \includegraphics[width=0.45\textwidth]{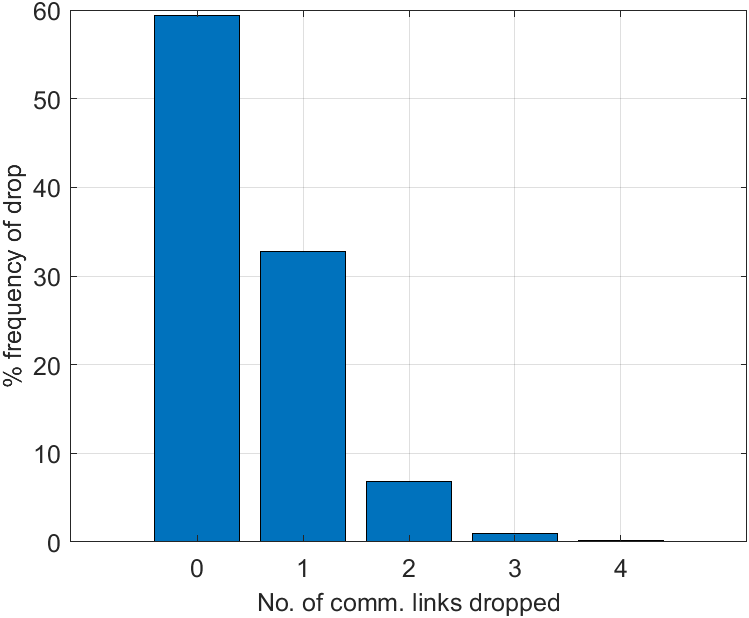}
    \caption{Percentage frequency of no. of communication links dropped over the course of 140 sec., with a link drop probability of 0.1, for a simulation run}
    \label{fig:03}
\end{figure}
The target and the robots follow the mathematical models as discussed in section II. The initial state values, control law parameters, drift terms in the prediction algorithm model, and the control input sequence for the target (considered unknown to the robots) are all set to suitable values. The loss function is defined to be $l(x,y) = \min (||x-y||/50 , 1)$, where $x,y \in \mathbb{R}^2$. Since $\tau = 1$, as per equations (\ref{eqt02.2}) and (\ref{eqt02.3}), the $i^{th}$ robot's control law uses its one-step look-ahead prediction of target's position at time $t$, $\hat{x}_{(t+1|t),Tg}^i$. As per equation (\ref{eqt04}), one-step look-ahead prediction by algorithm $A_i$ is given as:
\begin{equation} \label{eqt004}
        \hat{x}_{(t+1|t),T_g}^{A_i} = x_{t+1,T_g} + \zeta_{t,i}^{1} + \nu_{t,i}^{1} 
\end{equation}
where $x_{t+1,T_g}$ is the target's true position vector at time $t+1$, and $\zeta_{t,i}^{1} \in \mathbb{R}^2$ represents drift in algorithm $A_i$'s one-step look-ahead prediction of target's position, and $\nu_{t,i}^{1} \in \mathbb{R}^2$ is zero-mean gaussian prediction noise with covariance $C_{t,i}^{1}$, at time $t$. The drift term $\zeta_{t,i}^{1} = c_{t,i}^{1} s_t$ follows the model given by equations (\ref{eqt04.01}) and (\ref{eqt05}), and the drift reset probability is set to be $p = 0.1$. Further, we consider the term $c_{t,i}^{1} = \gamma_{t,i} [1, 1]'$, and $\nu_{t,i}^{1}$'s covariance term $C_{t,i}^{1} = (10 \cdot \gamma_{t,i})^2 diag([1,1])$, where the terms $\gamma_{t,i}$, $\forall i \in [N]$, vary with time as shown in Table \ref{tab:tbl003}. $\gamma_{t,i}$ values are indicative of how good or bad algorithm $A_i$ is at time $t$; larger $\gamma_{t,i}$ values lead to a lower prediction accuracy. From table \ref{tab:tbl003}, note the variation in the $\gamma_{t,i}$ over the horizon of $T=1400$ discrete-time steps; for instance, algorithm $A_1$ (installed in the $1^{st}$ robot) is accurate initially but its prediction degrades later on, whereas the opposite can be said about algorithm $A_6$ (installed in the $6^{th}$ robot). 
\begin{table}[H] 
    \centering
    \begin{tabular}{|c|c|c|c|c|c|c|}
    \hline
    \hline
    $\gamma_{t,i}$   & $i=1$ & $i=2$ & $i=3$ & $i=4$ & $i=5$ & $i=6$ \\
    \hline 
    \hline 
    $\gamma_{(1:T/6),i}$ & $0.01$  &  $0.1$ &  $0.1$ &  $0.2$ &  $0.4$ &  $0.8$ \\
    \hline
    $\gamma_{(T/6:T/3),i}$ & $0.01$  &  $0.1$ &  $0.2$ &  $0.1$ &  $0.3$ &  $0.6$ \\
    \hline     
    $\gamma_{(T/3:T/2),i}$ & $0.3$  &  $0.3$ &  $0.4$ &  $0.05$ &  $0.3$ &  $0.3$ \\
    \hline     
     $\gamma_{(T/2:5T/6),i}$ & $0.6$  &  $0.3$ &  $0.2$ &  $0.2$ &  $0.3$ &  $0.01$ \\
    \hline 
    $\gamma_{(5T/6:T),i}$ & $0.8$  &  $0.3$ &  $0.2$ &  $0.2$ &  $0.1$ &  $0.01$ \\
    \hline
    \end{tabular}
    \caption{Drift and Noise proportionality term $\gamma_{t,i}$ for different time duration, where $c_{t,i}^1 = \gamma_{t,i} [1,1]'$ and $C_{t,i}^{1} = (10 \cdot \gamma_{t,i})^2 diag([1,1])$, $i = 1,2,\cdots,6$.}
    \label{tab:tbl003}
\end{table}

For the above described simulation setup with an adverse setting, D2EAL is compared against three baseline (Mean, Median, Greedy-Local) and four state-of-the-art (Kalman Fusion, Bayes Fusion, Covariance Intersection, Covariance Union) decentralized prediction/estimate fusion methods, and the case with no communication among the robots, which are briefly described as follows: 
\begin{itemize}
    \item \textit{No Communication}: involves no communication among the robots; the robots just rely on their respective prediction algorithms for one-step look-ahead prediction of target's position.
    \item \textit{Greedy-Local}: each robot directly uses the one-step look-ahead prediction $\hat{f}_{t+1,i} = f_{t+1,j_*}$, which incurs the least cumulative loss among all the predictions that are shared by its neighbours and its own prediction algorithm, i.e., $j_* = \arg \min_{j\in \Lambda_i(t)} L_{t,j}$, $\forall i \in [N]$. 
    \item \textit{Mean}: each robot takes the mean of all the predictions shared by its neighbours and its own prediction algorithm, i.e., $\hat{f}_{t+1,i} = \frac{1}{d_i(t)} \sum_{\forall j \in \Lambda_i(t)} f_{t+1,j}$, where $d_i(t) = \sum_{\forall j \in  \Lambda_i(t)} (1)$, $\forall i \in [N]$. 
    \item \textit{Median}: Instead of mean, each robot takes the median of all the predictions shared by its neighbours and its own prediction algorithm. 
    \item \textit{Kalman Fusion} (\cite{maybeck1982stochastic},\cite{uhlmann2003covariance}): each robot takes the Kalman Fusion of all the predictions given by its own prediction algorithm and that of its neighbours; assumes that the predictions being fused are uncorrelated and their associated zero-mean gaussian noises' covariance ($C_{t,i}^1$) are known.
    \item \textit{Bayes Fusion} (\cite{weng2012bayesian}): each robot employs a Bayesian framework for the fusion of all the predictions given by its own prediction algorithm and that of its neighbours; assumes that the predictions being fused are consistent and their associated zero-mean gaussian noises' covariance ($C_{t,i}^1$) are known, but their cross-correlation is unknown. 
    \item \textit{Covariance Intersection} (\cite{matzka2009comparison},\cite{julier2017general}): each robot employs the Covariance Intersection method for the fusion of all the predictions given by its own prediction algorithm and that of its neighbours; assumes that the predictions being fused are consistent and their associated zero-mean gaussian noises' covariance ($C_{t,i}^1$) are known, but their cross-correlation is unknown.
    \item \textit{Covariance Union} (\cite{matzka2009comparison},\cite{reece2010generalised}): each robot employs the Covariance Union method for the fusion of all the predictions given by its own prediction algorithm and that of its neighbours; assumes that the predictions being fused can be  inconsistent and their cross-correlation is unknown, but their associated zero-mean gaussian noises' covariance ($C_{t,i}^1$) are known. 
\end{itemize}
For D2EAL algorithm, the learning parameters are set to be $\eta_{\alpha} = 2$ and $\eta_{w} = 2$ via trail and error. Both D2EAL and Greedy-Local involve a periodic reset for their weights and cumulative loss variable, respectively, after every $T_o = 200$ discrete time steps. Note that D2EAL, Greedy-Local, Mean, and Median do not require covariance information of the predictions as input. 

Snapshots for the D2EAL simulation case are shown in fig.\ref{fig:04} for four different time instants. As we can see in the figure, the six robots successfully chase the target while maintaining some distance from the target and from each other, in addition to following the heading angle requirement of making sure that their heading directions point towards next-step position of the target with sufficient accuracy. In fact, this improvement in accuracy of the one-step look-ahead prediction of target's position by all the robots has occurred due to the use of D2EAL algorithm, as is quite evident from fig.\ref{fig:05}. Note that cumulative loss for $i^{th}$ robot is given as $\hat{L}_{t,i} = \sum_{s=1}^{t} \hat{l}_{s,i}$. Fig.\ref{fig:05} shows how $6^{th}$ robot's cumulative loss ($\hat{L}_{t,6}$), averaged over $100$ simulation runs, evolves with time. Note the relative improvement in the performance of the $6^{th}$ robot for the case with D2EAL, compared to the other fusion algorithms -- approximately $44\%$ improvement compared to the best performing covariance-based method, Bayes Fusion (BF). Note that before $t=70$ sec., algorithm $A_6$ is quite inaccurate whereas its accuracy increases after $t=70$ sec., as shown in Table \ref{tab:tbl003}, which is also reflected in the plot of $6^{th}$ robot's cumulative loss in fig.\ref{fig:05}. Similar plots can be shown for other robots as well. This shows the effectiveness of D2EAL in making sure that each robot incurs sufficiently smaller prediction losses irrespective of how bad its prediction algorithm's or that of its neighbours' prediction algorithm's performance is. Fig.\ref{fig:05} also shows the evolution of total cumulative loss of all the robots ($\sum_{i=1}^{6} \hat{L}_{t,i}$) with respect to time, averaged over 100 simulation runs. For D2EAL, the total cumulative loss incurred by all the robots is significantly lesser compared to the other fusion algorithms -- around $30\%$ lesser compared to BF. This shows that D2EAL enables each robot to benefit from the robot having the best prediction information irrespective of its placement in the communication network, thus lowering the overall total cumulative loss. 
\begin{figure}[h]
    \centering
    \includegraphics[width=0.48\textwidth]{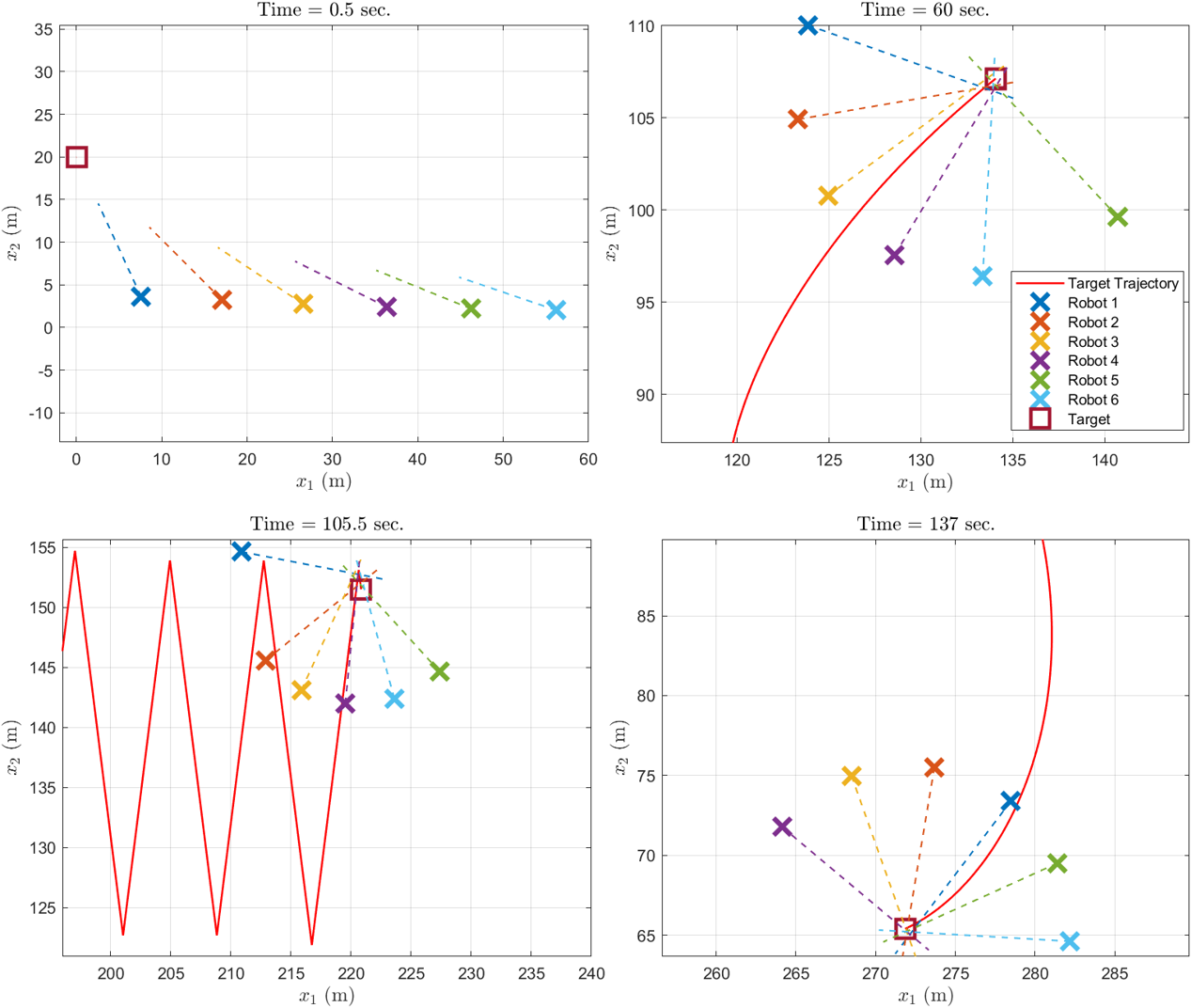}
    \caption{Snapshots of D2EAL Simulation at time instants 0.5 sec., 60 sec., 105.5 sec., and 137 sec., respectively. The dashed lines represent robots' heading direction.}
    \label{fig:04}
\end{figure}
\begin{figure}[h]
    \centering
    \includegraphics[width=0.48\textwidth]{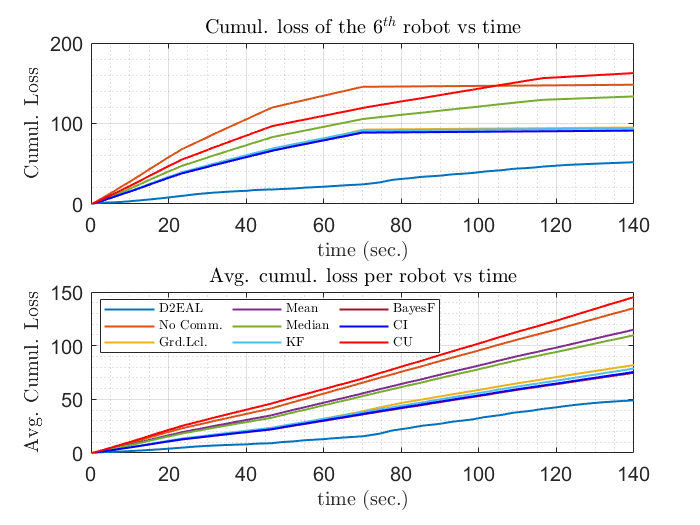}
    \caption{Cumulative loss of the $6^{th}$ robot as a function of time, and total cumulative loss of all the robots as a function of time, respectively, averaged over 100 simulation runs, with $N = 6$.}
    \label{fig:05}
\end{figure}

In the second simulation study, D2EAL is evaluated for its scalability in terms of total average cumulative loss incurred per robot at the end of horizon $T$ versus total no. of robots $N$. For the scalability study, starting with the case of $N=2$, where one robot's prediction algorithm is quite accurate ($\gamma_{t,1} = 0.01$) and the other one's prediction algorithm is inaccurate ($\gamma_{t,2} = 0.8$), we keep on adding new robots in-between the originally chosen two robots in the underlying linear graph, such that for the new robots, $\gamma_{t,i} = unif(0,2)*\frac{0.01+0.8}{2}$, where $unif(0,2)$ is a uniform random variable within the range $[0,2]$. This makes sure that for the case of no communication, the average cumulative loss per robot ($\frac{1}{N}\sum_{i=1}^{N} \hat{L}_{t,i}$) always stays $\approx175$ as $N$ is increased. Fig. \ref{fig:06} shows the plot for average cumulative loss per robot versus the total no. of robots ($N$). Note that D2EAL outperforms all the other fusion algorithms in the scalability test as well, since its average cumulative loss per robot stays substantially lower (approx. $32\%$ w.r.t. BF) than that of other algorithms as $N$ is increased. Fig. \ref{fig:06} also shows the plot for reliability cost versus total no. of robots for a typical MRS, where reliability cost is considered to be inversely proportional to the total no. of robots $N$. This is justified since increasing the total no. of robots increases a MRS's fault tolerance capability, especially in case of \textit{best} robot's failure, either in terms of its prediction or communication. As $N$ is increased, we can observe that the average cumulative prediction loss per robot stays almost the same, while the reliability cost decreases. This way, D2EAL allows a MRS to exhibit high reliability by increasing $N$, while making sure that prediction performance doesn't degrade as $N$ is increased. 
\begin{figure}[h]
    \centering
    \includegraphics[width=0.48\textwidth]{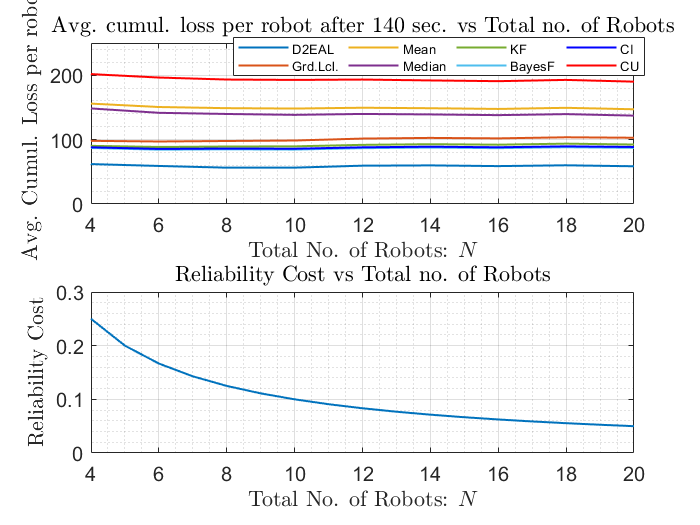}
    \caption{The first plot shows the average cumulative loss per robot for a horizon of 140 sec. as a function of total no. of robots $N$, averaged over 100 simulation runs. The second plot shows how the reliability cost of a typical multi-robot system varies with $N$.}
    \label{fig:06}
\end{figure}

From Fig. \ref{fig:05} and Fig. \ref{fig:06}, it is quite evident that D2EAL performs substantially better than the baseline as well as the state-of-the-art fusion methods. Note that the performance of Kalman Fusion (KF), Covariance Intersection (CI), and Bayes Fusion (BF) is comparable to each other, with BF slightly better than CI, which confirms with the results in \cite{weng2012bayesian}. Also, the performance of CI is slightly better than KF; CI does not assume that the predictions being fused are uncorrelated, whereas KF does. Further, Greedy-Local, a method that does not require the knowledge of prediction covariance, performs slightly worse than these three covariance-based methods as discussed above. The performance of Greedy-Local is significantly better than that of Mean and Median, which perform comparable to each other. The performance of Covariance Union (CU) is the worst among all the methods; it is even worse than the no communication case. This can attributed to the fact that in CU, while trying to keep the fused predictions consistent, the resultant fused covariance is increased - this keeps on increasing the resultant covariance of the fused (output) prediction which leads to even more inaccuracy than the input predictions. Note that D2EAL does not require the knowledge of prediction covariance. Still, D2EAL is able to outperform these covariance-based methods because of its online learning of weights based on a prediction loss feedback, combined with its two weighted fusion phases.             

\section{Conclusion}
This paper presents a novel Decentralized Distributed Expert-Assisted Learning (D2EAL) algorithm for cooperative target tracking using a heterogeneous MRS. D2EAL involves each robot running a two-layered exponentially weighted information fusion process combined with the online learning of weights based on a prediction loss feedback, thereby enabling each robot to improve its own look-ahead prediction performance as well as that of its neighbours in the communication network. D2EAL's theoretical analysis shows that under certain reasonable assumptions, the worst-case bounds on the cumulative losses grow sub-linearly with the horizon $T$, and the weights do converge as well. Simulation results show that, in an adverse setting involving large dynamic drift/bias in the predictions, D2EAL outperforms the three baseline and four well-known decentralized fusion methods considered for comparison. D2EAL is shown to be superior to these seven methods in terms of scalability as well. In both the simulation studies, D2EAL performs approximately $30\%$ better than the best performing covariance-based method, Bayes Fusion. Moreover, as $N$ is increased, the average cumulative loss per robot for D2EAL stays plateaued while the reliability cost decreases, thus showing how D2EAL's scalability allows for higher reliability in the multi-robot system. Further, note that D2EAL algorithm involves analytic expressions, which makes it computationally inexpensive and easy to implement. The current problem formulation can be extended to the case where the target is partially observable to the multi-robot system, which needs further investigation. 

\bibliographystyle{unsrt}  
\bibliography{references}

\end{document}